\documentclass[onecolumn]{article}
\usepackage{epsfig}
\usepackage{booktabs}
\usepackage{amssymb,amsmath,amsthm,multirow}
\usepackage{amsfonts,graphicx,subfigure}
\usepackage{float}
\usepackage{caption}
\usepackage{subcaption}
\usepackage{algorithm}
\usepackage{algpseudocode}
\usepackage{booktabs}
\usepackage{multirow}
\usepackage{rotating}
\newtheorem{theorem}{Theorem}
\newtheorem{cor}{Corollary}
\newtheorem{lemma}{Lemma}
\newtheorem{proposition}{Proposition}
\newtheorem{dfn}{Definition}

\begin{document}
\title{Cluster-Adaptive Feature Extraction and its Theoretical Foundation with Minkowski Weighted k-Means}
\author{Renato Cordeiro de Amorim\thanks{School of Computer Science and
Electronic Engineering, University of Essex, Wivenhoe, UK. r.amorim@essex.ac.uk} \and Vladimir Makarenkov\thanks{D\'epartement d'informatique, Universit\'e du Québec \`a Montr\'eal, C.P. 8888 succ. Centre-Ville, Montreal (QC) H3C 3P8 Canada.} \thanks{Mila - Quebec AI Institute, Montreal, QC, Canada.}} 
\date{}
\maketitle

\begin{abstract}
The Minkowski weighted $k$-means ($mwk$-means) algorithm extends classical $k$-means by incorporating feature weights and a Minkowski distance. We first show that the $mwk$-means objective can be expressed as a power-mean aggregation of within-cluster dispersions, with the order determined by the Minkowski exponent $p$. This formulation reveals how $p$ controls the transition between selective and uniform use of features. Using this representation, we derive bounds for the objective function and characterise the structure of the feature weights, showing that they depend only on relative dispersion and follow a power-law relationship with dispersion ratios. This leads to explicit guarantees on the suppression of high-dispersion features, and we establish convergence of the algorithm. Building on these theoretical results, we introduce Cluster-Adaptive Feature Extraction (CAFE), a method that uses the $mwk$-means feature weights to rescale the data prior to unsupervised feature extraction. We prove that this rescaling reverses the within-cluster dispersion ordering, suppressing noisy features and amplifying informative ones. Numerous experiments conducted under controlled within-cluster noise show that CAFE consistently improves the results of traditional feature extraction methods.
\end{abstract}
\vspace{1em}

\section{Introduction}
Clustering is a fundamental task in unsupervised learning. It aims to partition a set of data points into groups (i.e., clusters) in such a way that points within the same group are similar and points in different groups are dissimilar. Clustering enables researchers to uncover patterns without requiring labelled samples (which may be expensive to acquire, or simply unavailable). Hence, it often forms a critical first step in exploratory data analysis and knowledge discovery pipelines. Clustering algorithms have been successfully applied across numerous research areas, including bioinformatics, image processing, social network analysis, and natural language processing \cite{lin2022clustering,lopez2025biclustering,pimenta2024androidgyny,ezugwu2022comprehensive,de2021identifying,huang2024deep}. 

The clustering literature presents many algorithms, among which $k$-means \cite{macqueen1967some} is arguably the most popular \cite{ikotun2023k,liu2023transforming,ahmed2020k,harris2022extensive}. Its popularity stems from its conceptual simplicity, computational efficiency, and strong performance across a wide range of practical applications. However, $k$-means has well-known limitations. In particular, it assumes that all features (i.e., the variables describing each data point) are equally important to the clustering process. This assumption is rarely satisfied in real-world applications, where even relevant features may exhibit substantially different degrees of relevance.

To address this limitation, several extensions of $k$-means have been proposed that incorporate feature weighting mechanisms \cite{lai2025silhouette,hancer2020survey,li2020fuzzy,oskouei2024feature}. These methods aim to automatically assign higher weights to more informative features, thereby improving clustering performance in data sets where feature relevance varies. One notable example is the Minkowski weighted $k$-means ($mwk$-means) algorithm \cite{de2012minkowski}, which generalises classical $k$-means by employing a weighted Minkowski distance and iterative weight updates to jointly learn cluster assignments and feature importances.

The $mwk$-means algorithm has demonstrated popularity and empirical success in various scenarios \cite{NINOADAN2021115424,aradnia2022adaptive,deng2016survey,melvin2016uncovering,melvin2017mutsalpha,jamali2020combination,gowthaman2025novel}, including its use as a tool to improve the estimation of the number of clusters in a data set \cite{de2015recovering}. However, it has received limited attention from a theoretical standpoint, leaving a gap in our understanding of its fundamental properties. Several theoretical questions remain insufficiently explored, including the convergence behaviour of the algorithm, the role of the Minkowski exponent in shaping the objective, and the structure of the resulting feature weights. A formal analysis of these aspects is crucial not only for establishing the method's theoretical soundness but also for guiding its application and further development.

This paper aims to bridge this gap by providing a rigorous theoretical study of $mwk$-means and introducing a new method that builds directly on its theoretical properties. We show that the $mwk$-means objective function can be expressed as a power-mean aggregation of within-cluster dispersions, thereby providing a unifying interpretation of the role of the Minkowski exponent. Building on this formulation, we analyse the properties of the objective and the induced feature weighting mechanism, establish convergence guarantees, and introduce Cluster-Adaptive Feature Extraction (CAFE), a method that exploits the $mwk$-means feature weights to improve unsupervised feature extraction in the presence of within-cluster noise.

In particular, our main contributions are as follows: (i) a reformulation of the $mwk$-means objective as a power-mean aggregation of within-cluster dispersions; (ii) theoretical bounds for the objective derived from power-mean inequalities; (iii) a characterisation of the structure and scaling behaviour of the feature weights; (iv) convergence guarantees for the algorithm; and (v) CAFE, a cluster-adaptive feature extraction method grounded in the theoretical properties of $mwk$-means, and empirical validation.

\section{Related work}

\subsection{Minkowski weighted $k$-means}

The $k$-means algorithm is a classical clustering method that aims to partition a data set \(X = \{x_1, \ldots, x_n\}\) into a clustering \(S = \{S_1, \ldots, S_k\}\) such that \(X = \bigcup_{l=1}^k S_l\), and \(S_l \cap S_t = \emptyset \) for all \(l \neq t\). The algorithm minimises the within-cluster sum of squared distances,
\begin{equation}
\label{eq:kmeans}
W(S,Z) = \sum_{l=1}^k \sum_{x_i \in S_l} \sum_{v=1}^m (x_{iv} - z_{lv})^2,
\end{equation}
where each \(x_i \in X\) is described over \(m\) features, and \(z_l \in Z\) is the centroid of cluster \(S_l\). The algorithm minimises~\eqref{eq:kmeans} iteratively via the following steps:
\begin{enumerate}
    \item Select \(k\) data points from \(X\) at random and set them as the initial centroids \(Z=\{z_1, \ldots, z_k\}\).
    \item Assign each \(x_i \in X\) to the cluster \(S_l\) whose centroid \(z_l\) is closest to \(x_i\).
    \item Update each centroid \(z_l\) as the component-wise mean over \(x_i \in S_l\).
    \item Repeat Steps (2) and (3) until the centroids no longer change.
\end{enumerate}

This iterative procedure is guaranteed to converge in a finite number of steps, as each iteration monotonically decreases the objective function~(\ref{eq:kmeans}). However, the algorithm is only guaranteed to find a local minimum, and its performance is sensitive to the initialisation of the centroids. The $k$-means algorithm implicitly assumes that all features are equally important, that clusters are roughly spherical and of similar size, and that the Euclidean distance is an appropriate measure of dissimilarity. These assumptions may not always hold, motivating various extensions such as feature-weighted and generalised distance clustering algorithms \cite{lai2025silhouette,hancer2020survey,li2020fuzzy,oskouei2024feature}.

The Minkowski weighted $k$-means ($mwk$-means) \cite{de2012minkowski} is a popular feature-weighted clustering algorithm \cite{NINOADAN2021115424,aradnia2022adaptive,aradnia2022adaptive,deng2016survey,melvin2016uncovering,melvin2017mutsalpha,jamali2020combination,gowthaman2025novel}, and has also been shown to improve the estimation of the correct number of clusters in a data set~\cite{de2015recovering}. This algorithm employs a weighted Minkowski distance,
\begin{equation}
    \label{eq:weighted_mink_dist}
    d(x_i, z_l) = \sum_{v=1}^m w_{lv}^p |x_{iv}-z_{lv}|^p,
\end{equation}
where \(p>1\) is the Minkowski exponent, and \(w_{lv}\) is the weight of feature \(v\) in cluster \(S_l\). In particular, $mwk$-means allows features to have different degrees of relevance across clusters, reflecting the intuition that feature importance may vary locally within the data. The objective function of the algorithm is given by

\begin{equation}
    \label{eq:mwk}
    W_p(S,Z,w) = \sum_{l=1}^k \sum_{x_i \in S_l} \sum_{v=1}^m w_{lv}^p |x_{iv} - z_{lv}|^p,
\end{equation}
which, when minimised with respect to the weights, leads to
\begin{equation}
    \label{eq:mwk_weight}
    w_{lv}= \frac{1}{\sum_{u=1}^m \left[\frac{D_{lv}}{D_{lu}}\right]^{\frac{1}{p-1}}},
\end{equation}
where \(D_{lv} = \sum_{x_i \in S_l} |x_{iv} - z_{lv}|^p\), and \(\sum_{v=1}^m w_{lv} = 1\) for each cluster \(S_l\). Algorithm~\ref{alg:mwkmeans} formally describes the iterative steps used to minimise~\eqref{eq:mwk}. An important observation is that the weight update rule in $mwk$-means can be viewed as a generalisation of feature selection. Traditional feature selection methods assign binary weights to features, effectively including or excluding them from the clustering process (i.e., weights of zero or one) \cite{li2017feature}. In contrast, feature-weighted clustering methods allow for varying degrees of importance across features, enabling the model to account for different levels of relevance even among informative features \cite{shang2023unsupervised}.

The Minkowski exponent \(p > 1\) plays a central role in shaping the distance metric and the resulting clustering geometry. When \(p = 2\), the Minkowski distance reduces to the squared Euclidean distance, and in the context of $mwk$-means this yields a feature-weighted variant of the classical \(k\)-means algorithm. For values of \(p \ne 2\), the geometry induced by the distance changes, leading to different cluster boundaries, and the algorithm generalises beyond the Euclidean setting. This flexibility allows $mwk$-means to adapt to a wider range of data distributions while also learning feature relevance.

Each centroid \(z_l\) in $mwk$-means is defined as the Minkowski centre of the points \(x_i \in S_l\). Specifically, for each feature \(v\), the coordinate \(z_{lv}\) minimises the function
\begin{equation}
    \label{eq:mink_centre}
    f_p(z) = \sum_{x_i \in S_l} |x_{iv} - z|^p.    
\end{equation}
This function has closed-form solutions for certain values of \(p\). When \(p = 2\), the minimum is achieved at the mean. As \(p \to \infty\), the solution approaches the midrange, i.e., the midpoint between the minimum and maximum values. For general \(p \in (1, \infty)\), however, there is no closed-form solution, and the minimiser must be computed numerically, typically via iterative methods such as gradient-based optimisation. The use of the Minkowski centre allows $mwk$-means to adapt to different distance geometries while jointly learning feature relevance. The formulation above provides the basis for the theoretical analysis developed in the next section.

\begin{algorithm}
    \caption{Minkowski weighted $k$-means ($mwk$-means)}
    \label{alg:mwkmeans}
    \begin{algorithmic}[1]
        \Require Data set $X$, number of clusters $k$, exponent $p$.
        \Ensure Clustering $S = \{S_1, \dots, S_k\}$, centroids, and weights.
        \State Set \(w_{lv}=\frac{1}{m}\) for \(l=1, \ldots, k\) and \(v=1, \ldots, m\).
        \State Select $k$ data points from $X$ uniformly at random, and copy their values into $z_1, \ldots, z_k$.
        \Repeat
            \State Assign each $x_i \in X$ to the cluster of its nearest centroid according to (\ref{eq:weighted_mink_dist}). That is,
            \[
            S_l \gets \{x_i \in X \mid l=\arg\min_{t} d(x_i, z_t)\}.
            \]
            \State Update each $z_{l} \in Z$ to the component-wise Minkowski centre of $x_i \in S_l$.
            \State Update each \(w_{lv}\) using (\ref{eq:mwk_weight}).
        \Until{the objective (\ref{eq:mwk}) converges.}
        \State \Return Clustering $C$ and centroids $Z$.
    \end{algorithmic}
\end{algorithm}

\subsection{Unsupervised feature extraction}

Feature extraction aims to map a high-dimensional data set $X \in \mathbb{R}^{n \times m}$ to a lower-dimensional representation $Y \in \mathbb{R}^{n \times r}$, with $r \ll m$, that retains the most relevant structure of the original data. This is particularly important in clustering, where the presence of irrelevant or noisy features can obscure cluster structure and degrade the performance of algorithms such as $k$-means. By reducing the dimensionality of the data prior to clustering, feature extraction methods can improve both the quality of the resulting partition and the computational efficiency of the clustering process~\cite{van2009dimensionality}.

Principal Component Analysis (PCA) is arguably the most widely used feature extraction method. It finds a set of $r$ orthogonal directions, known as principal components, that maximise the variance of the projected data. Formally, the $j$-th principal component is the eigenvector corresponding to the $j$-th largest eigenvalue of the covariance matrix of $X$. PCA is computationally efficient and well understood theoretically, and its connection to $k$-means has been established by Ding and He~\cite{ding2004k}, who showed that the cluster membership indicators of $k$-means lie in the subspace spanned by the leading principal components.

Non-negative Matrix Factorisation (NMF)~\cite{seung2001algorithms} seeks a decomposition $X \approx FG$, where $F \in \mathbb{R}^{n \times r}$ and 
$G \in \mathbb{R}^{r \times m}$ are both constrained to have non-negative entries. The non-negativity constraint encourages a parts-based representation of the data, which has been shown to be beneficial for clustering tasks where the data are naturally non-negative, such as text or image data~\cite{wang2012nonnegative}. In practice, NMF is sensitive to initialisation, and multiple restarts are typically used to mitigate the effect of local minima.

Independent Component Analysis (ICA)~\cite{le2011ica} seeks a linear transformation of the data such that the resulting components
are as statistically independent as possible. Unlike PCA, which finds uncorrelated components, ICA imposes the stronger condition of statistical independence, making it better suited to separating mixed signals with non-Gaussian distributions~\cite{le2011ica}. ICA is widely used in signal processing and has been applied to clustering as a preprocessing step.

Autoencoders~\cite{hinton2006reducing} are neural network models that learn a compact representation of the data through a bottleneck architecture. An encoder maps the input to a lower-dimensional latent space, and a decoder attempts to reconstruct the original input from this representation. The latent representation is trained to minimise the reconstruction error, encouraging it to capture the most informative aspects of the data. In the context of clustering, the latent representation can be used as a feature space for downstream methods such as $k$-means.

Uniform Manifold Approximation and Projection (UMAP)~\cite{mcinnes2018umap} is a non-linear dimensionality reduction method grounded in Riemannian geometry and algebraic topology. It constructs a fuzzy topological representation of the data in the high-dimensional space and seeks a low-dimensional embedding that preserves this structure. UMAP is particularly effective at preserving both local and global structure in the data, and has demonstrated strong empirical performance across a wide range of applications.

\section{Theoretical analysis}
We now provide a theoretical analysis of the $mwk$-means objective. Building on the formulation introduced in Section~2, we examine its convergence properties, derive an equivalent representation in terms of within-cluster dispersions, and analyse the structure of the induced feature weights. This analysis provides insight into the role of the Minkowski exponent and its effect on the behaviour of the algorithm.

\subsection{Convergence Properties}
In this section, we establish the convergence properties of the $mwk$-means algorithm. As with classical $k$-means, the algorithm proceeds by alternating optimisation over cluster assignments, centroids, and feature weights. We show that each of these updates is well-defined and leads to a monotonic decrease of the objective function, which in turn guarantees convergence.

A key component of the algorithm is the update of cluster centroids, which are defined as Minkowski centres. The following result ensures that this update is well-posed.

\begin{proposition}
\label{prop:mink_centre_convex}
If \(p>1\), then \(f_p(z)\) admits a unique Minkowski centre.
\end{proposition}
\begin{proof}
The function \(t \mapsto |t|^p\) is strictly convex on \(\mathbb{R}\) for \(p > 1\). Hence, for each fixed \(x_{iv}\), the function
\(
z \mapsto |x_{iv}-z|^p
\)
is strictly convex in \(z\), being the composition of a strictly convex function with an affine map. Therefore, \(f_p(z)\), as a finite sum of strictly convex functions, is itself strictly convex. A strictly convex function has at most one minimiser. Moreover, \(f_p(z)\to\infty\) as \(|z|\to\infty\), so \(f_p(z)\) is coercive. Thus, it attains its minimum, which is unique.
\end{proof}

Proposition~\ref{prop:mink_centre_convex} guarantees that, for $p>1$, each cluster admits a unique Minkowski centre. This property is essential, as it ensures that the centroid update step is unambiguous and that the objective is minimised with respect to $Z$ at each iteration. We can now establish the convergence of the overall algorithm.

\begin{theorem}
For any fixed \( p > 1 \), $mwk$-means monotonically decreases its objective function at each iteration and converges in a finite number of steps.
\end{theorem}
\begin{proof}
The $mwk$-means algorithm minimises~\eqref{eq:mwk} by alternating updates over cluster assignments, centroids, and weights. Reassigning points to their nearest centroid minimises~\eqref{eq:mwk} with respect to $S$. Updating each centroid to the Minkowski centre of its cluster minimises~\eqref{eq:mwk} with respect to $Z$ (which is unique by Proposition~\ref{prop:mink_centre_convex}). Updating weights via~\eqref{eq:mwk_weight} minimises~\eqref{eq:mwk} with respect to $w$ under the normalisation constraint on the weights. Hence, $W_p^{(t+1)} \leq W_p^{(t)}$ at every iteration $t$.

Since the number of possible clusterings is finite, and $W_p^{(t+1)} = W_p^{(t)}$ implies convergence, the algorithm terminates after a finite number of steps.
\end{proof}
The above theorem shows that $mwk$-means inherits the fundamental convergence properties of classical $k$-means. In particular, each iteration decreases the objective function, and since the number of possible clusterings is finite and the objective is lower bounded, the algorithm must terminate after a finite number of steps. As in the standard $k$-means setting, the algorithm converges to a local minimum of the objective, with the final solution depending on the initialisation.

\subsection{Bounds for the objective}
In this section, we derive bounds for the $mwk$-means objective function~\eqref{eq:mwk}. Our strategy is to reformulate the objective in terms of within-cluster dispersions, which enables a direct analytical characterisation of its behaviour. This representation reveals that the objective depends solely on the dispersion structure of each cluster, rather than explicitly on the feature weights.

\begin{lemma}
    \label{lemma:mwkmeans_in_terms_of_D}
    The $mwk$-means objective function~\eqref{eq:mwk} can be expressed in terms of within-cluster dispersions as \[ W_p(S,Z)=\sum_{l=1}^k \frac{1}{\left(  \sum_{v=1}^m \frac{1}{D_{lv}^{\frac{1}{p-1}}}\right)^{p-1}}.\]
\end{lemma}
\begin{proof}
    Recall that \(D_{lv}\) is the dispersion of feature \(v\) at cluster \(S_l\). That is, \(D_{lv} = \sum_{x_i \in S_l} |x_{iv} - z_{lv}|^p\). Then,
    \begin{align*}
    W_p(S,Z,w) = \sum_{l=1}^k \sum_{x_i \in S_l} \sum_{v=1}^m w_{lv}^p |x_{iv} - z_{lv}|^p=\sum_{l=1}^k \sum_{v=1}^m w_{lv}^p \sum_{x_i \in S_l} |x_{iv} - z_{lv}|^p = \sum_{l=1}^k \sum_{v=1}^m w_{lv}^p D_{lv}.
    \end{align*}
    Substituting~\eqref{eq:mwk_weight} into the above leads to
    \begin{align*}
    W_p(S,Z,w)  &=\sum_{l=1}^k \sum_{v=1}^m \left(\frac{1}{\sum_{u=1}^m \left[\frac{D_{lv}}{D_{lu}}\right]^{\frac{1}{p-1}}} \right)^p D_{lv}=\sum_{l=1}^k \sum_{v=1}^m \frac{D_{lv}}{D_{lv}^{\frac{p}{p-1}}\left(\sum_{u=1}^m \left[\frac{1}{D_{lu}}\right]^{\frac{1}{p-1}}\right)^p}\\
    &=\sum_{l=1}^k \sum_{v=1}^m \frac{D_{lv}^{1-\frac{p}{p-1}}}{\left(\sum_{u=1}^m \left[\frac{1}{D_{lu}}\right]^{\frac{1}{p-1}}\right)^p}=\sum_{l=1}^k \frac{1}{\left(\sum_{u=1}^m D_{lu}^{-\frac{1}{p-1}}\right)^p} \sum_{v=1}^m D_{lv}^{-\frac{1}{p-1}}.\\
    \end{align*}
    Clearly, \(\sum_{u=1}^m D_{lu}^{-\frac{1}{p-1}} = \sum_{v=1}^m D_{lv}^{-\frac{1}{p-1}}\). Hence,
    \begin{align*}
        \sum_{l=1}^k \frac{1}{\left(\sum_{u=1}^m D_{lu}^{-\frac{1}{p-1}}\right)^p} \sum_{v=1}^m D_{lv}^{-\frac{1}{p-1}} = \sum_{l=1}^k \frac{1}{\left(  \sum_{v=1}^m \left[\frac{1}{D_{lv}}\right]^{\frac{1}{p-1}}\right)^{p-1}}.
    \end{align*}    
\end{proof}

The above result provides a compact expression of the objective entirely in terms of the dispersions $\{D_{lv}\}$. This form is particularly convenient, as it separates the contribution of each cluster and removes the explicit dependence on the weights. As a consequence, the behaviour of $mwk$-means can be analysed through properties of the dispersion values alone.

To further interpret this expression, we now relate it to the family of power means. This connection allows us to leverage well-known inequalities and ordering properties of means to establish bounds for the objective.

\begin{dfn}
Let $r \in \mathbb{R}$ and $D_{l1},\dots,D_{lm} > 0$. The power mean of order $r$ is defined as
\[
M_r(D_{l1},\dots,D_{lm})=\left(\frac{1}{m}\sum_{v=1}^m D_{lv}^r\right)^{1/r}.
\]
\end{dfn}

The following lemma shows that the $mwk$-means objective is proportional to a sum of power means of the within-cluster dispersions, with the order of the mean determined by the Minkowski exponent $p$.

\begin{lemma}
\label{lemma:power_mean}
Let each $D_{lv} > 0$. For any $p > 1$, the $mwk$-means objective function 
satisfies
\[
W_p(S,Z) = \frac{1}{m^{p-1}} \sum_{l=1}^k M_r(D_{l1}, \dots, D_{lm}),
\]
where $r = -1/(p-1)$.
\end{lemma}

\begin{proof}
Let $M_{rl} := M_r(D_{l1},\dots,D_{lm})$. By the definition of the power mean,
\[
    mM_{rl}^r=\sum_{v=1}^m D_{lv}^r.
\]
Hence, Lemma~\ref{lemma:mwkmeans_in_terms_of_D} gives
\begin{align*}
W_p(S,Z)&=\sum_{l=1}^k \frac{1}{\left(\sum_{v=1}^m D_{lv}^{-\frac{1}{p-1}}\right)^{p-1}} = \sum_{l=1}^k \frac{1}{\left(\sum_{v=1}^m D_{lv}^r\right)^{p-1}}\\
&=\sum_{l=1}^k \frac{1}{\left(m\,M_{rl}^r\right)^{p-1}} = \sum_{l=1}^k\frac{1}{m^{p-1} M_{rl}^{r(p-1)}}.
\end{align*}
Since $r(p-1)=-1$, this simplifies to
\[
W_p(S,Z)=\frac{1}{m^{p-1}}\sum_{l=1}^k M_{rl}.
\]
\end{proof}

This representation provides an important insight: the role of $p$ is to control how dispersion values are aggregated across features. In particular, since $r=-1/(p-1)<0$, the objective emphasises smaller dispersion values more strongly than larger ones. As $p$ varies, the aggregation transitions between different regimes of sensitivity to feature-wise dispersion.

An immediate consequence of this formulation is that minimising the $mwk$-means objective is equivalent to minimising a sum of power means of dispersions.

\begin{cor}
Let $p>1$ be fixed. Minimising the $mwk$-means objective $W_p(S,Z)$
is equivalent to minimising
\[
\sum_{l=1}^k M_r(D_{l1},\dots,D_{lm}),
\]
where $r=-1/(p-1)$.
\end{cor}

We are now in a position to derive explicit bounds for the objective. These follow directly from classical inequalities relating power means of different orders.

\begin{theorem}
\label{thm:mwk_bounds}
Let $D_{lv} > 0$ for all $l$ and $v$, and let $p>1$. Then the $mwk$-means 
objective satisfies
\[
\frac{1}{m^{p-1}}\sum_{l=1}^k \min_v D_{lv}
\;\le\;
W_p(S,Z)
\;\le\;
\frac{1}{m^{p-1}}\sum_{l=1}^k 
\left(\prod_{v=1}^m D_{lv}\right)^{1/m}.
\]
The lower bound is attained as $p \to 1^+$ and the upper bound as 
$p \to \infty$.
\end{theorem}

\begin{proof}
From Lemma~\ref{lemma:power_mean},
\[
W_p(S,Z)=\frac{1}{m^{p-1}}\sum_{l=1}^k M_r(D_{l1},\dots,D_{lm}),
\]
where $r=-1/(p-1)$. Since $p>1$, we have $r<0$. The power means satisfy the 
ordering
\[
\min_v D_{lv}
\le
M_r(D_{l1},\dots,D_{lm})
\le
M_0(D_{l1},\dots,D_{lm}),
\]
where $M_0$ denotes the geometric mean. Hence
\[
\min_v D_{lv}
\le
M_r(D_{l1},\dots,D_{lm})
\le
\left(\prod_{v=1}^m D_{lv}\right)^{1/m}.
\]

Multiplying by $1/m^{p-1}$ and summing over $l=1,\dots,k$ yields the stated 
bounds for $W_p(S,Z)$.

Finally, as $p\to1^+$ we have $r\to-\infty$ and $M_r\to\min_v D_{lv}$, while 
as $p\to\infty$ we have $r\to0$ and $M_r\to M_0$, the geometric mean. 
\end{proof}

The bounds in Theorem~\ref{thm:mwk_bounds} provide a clear characterisation of the behaviour of the $mwk$-means objective. The lower bound corresponds to the minimum dispersion within each cluster, while the upper bound corresponds to the geometric mean of dispersions. Importantly, these bounds are tight in the limiting cases of $p$. As $p \to 1^+$, the objective approaches a form that depends only on the smallest dispersion values, effectively emphasising the most compact features. In contrast, as $p \to \infty$, the objective approaches the geometric mean, yielding a more balanced contribution across features. This illustrates how the parameter $p$ governs the trade-off between feature selectivity and uniformity in the clustering process.

\subsection{Structure and Scaling of Feature Weights}

In this section, we characterise the structural properties of the feature weights induced by the $mwk$-means objective. We quantify how relative differences in dispersion are translated into weight ratios, and study how this mapping is modulated by the Minkowski exponent $p$.

We begin by expressing the weights in a normalised form that isolates their dependence on the dispersions.

\begin{proposition}
\label{prop:weight_structure}
For any cluster $S_l$ and features $u,v$,
\[
w_{lv} = \frac{D_{lv}^{-\frac{1}{p-1}}}{\sum_{t=1}^m D_{lt}^{-\frac{1}{p-1}}}
\quad \text{and} \quad
\frac{w_{lv}}{w_{lu}} = \left(\frac{D_{lu}}{D_{lv}}\right)^{\frac{1}{p-1}}.
\]
In particular,
\[
D_{lv} < D_{lu} \;\Longleftrightarrow\; w_{lv} > w_{lu}.
\]
\end{proposition}
\begin{proof}
The first expression follows by algebraic rearrangement of (\ref{eq:mwk_weight}). For the ratio, the common denominator cancels, giving
\[
\frac{w_{lv}}{w_{lu}}
=
\frac{D_{lv}^{-\frac{1}{p-1}}}{D_{lu}^{-\frac{1}{p-1}}}
=
\left(\frac{D_{lu}}{D_{lv}}\right)^{\frac{1}{p-1}}.
\]

Finally, since $p>1$, the function $t \mapsto t^{-\frac{1}{p-1}}$ is strictly decreasing for $t>0$. Hence,
\[
D_{lv} < D_{lu} \;\Longleftrightarrow\; w_{lv} > w_{lu}.
\]
\end{proof}

Proposition~\ref{prop:weight_structure} shows that the weighting scheme is entirely governed by relative dispersion, rather than absolute scale. In particular, the ratio of any two weights depends only on the corresponding ratio of dispersions, implying that the weighting mechanism is invariant to uniform rescaling of the data. 

This representation also enables a precise characterisation of the limiting behaviour of the weights.

\begin{cor}
As $p \to 1^+$, the weights concentrate on the set of features attaining the minimum dispersion, while all other weights vanish.
\end{cor}
\begin{proof}
From Proposition~\ref{prop:weight_structure},
\[
\frac{w_{lu}}{w_{lv}} = \left(\frac{D_{lv}}{D_{lu}}\right)^{\frac{1}{p-1}}.
\]
As $p \to 1^+$, we have $\frac{1}{p-1} \to \infty$. Hence, if $D_{lv} < D_{lu}$,
\[
\frac{w_{lu}}{w_{lv}} \to 0.
\]
which implies that only features with minimal dispersion retain a weight not tending to zero.
\end{proof}

The above result highlights a transition to a sparse regime: as $p \to 1^+$, the weighting mechanism increasingly concentrates mass on the most compact features. In this limit, the algorithm effectively disregards all features except those attaining minimal dispersion, yielding behaviour analogous to hard feature selection.

We now quantify how relative differences in dispersion translate into relative differences in weights.

\begin{theorem}
Let $p>1$. For any cluster $S_l$ and features $u,v$, if
\[
D_{lu} \ge C \, D_{lv}
\quad \text{for some } C>1,
\]
then
\[
w_{lu} \le C^{-\frac{1}{p-1}} \, w_{lv}.
\]
\end{theorem}

\begin{proof}
From Proposition~\ref{prop:weight_structure}, we have
\[
\frac{w_{lu}}{w_{lv}} = \left(\frac{D_{lv}}{D_{lu}}\right)^{\frac{1}{p-1}}.
\]
If $D_{lu} \ge C \, D_{lv}$, then
\[
\frac{D_{lv}}{D_{lu}} \le \frac{1}{C},
\]
and hence
\[
\frac{w_{lu}}{w_{lv}} \le C^{-\frac{1}{p-1}}.
\]
The result follows.
\end{proof}

This bound shows that weight suppression follows a power-law relationship with respect to dispersion ratios. The exponent $1/(p-1)$ controls the sensitivity of this mapping, with smaller values of $p$ amplifying differences between features and larger values attenuating them.

The dependence on $p$ can be made explicit by examining how these ratios evolve as the exponent varies.
\begin{cor}
The ratio
\[
\frac{w_{lu}}{w_{lv}} = \left(\frac{D_{lv}}{D_{lu}}\right)^{\frac{1}{p-1}}
\]
is monotone in $p$ and converges to $1$ as $p \to \infty$. Consequently, the weights become more uniform as $p$ increases, converging to $\frac{1}{m}$.
\end{cor}

This result shows that $p$ acts as a smoothness parameter controlling the contrast of the weighting scheme. As $p$ increases, the relative differences between weights are progressively reduced, leading to a uniform allocation in the limit.

Finally, we derive a global bound that characterises the suppression of features whose dispersion is consistently larger than that of all others.

\begin{theorem}
\label{thm:suppression}
Let $p>1$. For any cluster $S_l$ and feature $u$, suppose
\[
D_{lu} \ge C \, D_{lv} \quad \text{for all } v \neq u,
\]
with $C>1$. Then
\[
w_{lu} \le \frac{1}{1 + (m-1) C^{\frac{1}{p-1}}}.
\]
\end{theorem}
\begin{proof}
From Proposition~\ref{prop:weight_structure}, we have
\[
w_{lu} = \frac{1}{1 + \sum_{v \neq u} \left(\frac{D_{lu}}{D_{lv}}\right)^{\frac{1}{p-1}}}.
\]
By assumption, $D_{lu} \ge C D_{lv}$ for all $v \neq u$, so
\[
\frac{D_{lu}}{D_{lv}} \ge C.
\]
Thus,
\[
\left(\frac{D_{lu}}{D_{lv}}\right)^{\frac{1}{p-1}} \ge C^{\frac{1}{p-1}}.
\]
Hence,
\[
\sum_{v \neq u} \left(\frac{D_{lu}}{D_{lv}}\right)^{\frac{1}{p-1}}
\ge (m-1) C^{\frac{1}{p-1}}.
\]
Substituting into the expression for $w_{lu}$ yields
\[
w_{lu} \le \frac{1}{1 + (m-1) C^{\frac{1}{p-1}}},
\]
which completes the proof.
\end{proof}

This bound formalises the robustness of $mwk$-means to irrelevant or noisy features. In particular, features that exhibit uniformly larger dispersion are guaranteed to receive exponentially smaller weights, with the rate of decay controlled by both the dispersion ratio and the number of features. This provides a theoretical explanation for the empirical effectiveness of $mwk$-means in settings with heterogeneous feature relevance.

\subsection{Illustration of Theoretical Properties}

In this section we illustrate the main theoretical results derived above. To do so, we first generated 10 data sets, each with 1,000 data points, four features, and three clusters. The clusters are spherical Gaussian mixtures with zero mean and unit variance. We then augmented each data set with four additional features composed of uniformly random values (noise features). All data sets were subsequently normalised according to
\[
x_{iv} = \frac{x_{iv} - \bar{x}_v}{\max_i x_{iv} - \min_i x_{iv}},
\]
where \(\bar{x}_v\) denotes the mean of feature \(v\) over the data set \(X\). We then applied the $mwk$-means algorithm to each data set for different values of the Minkowski exponent \(p \in \{1.1, 1.5, 2, 5\}\), with 20 random initialisations per data set.

Figure~\ref{fig:weights_stability} illustrates the effect of the Minkowski exponent \(p\) on the distribution of feature weights. As predicted by Proposition~\ref{prop:weight_structure}, smaller values of \(p\) amplify differences in dispersion, leading to a sparse allocation of weights concentrated on a few features. In contrast, larger values of \(p\) attenuate these differences, resulting in a more uniform distribution of weights across features. This behaviour highlights the role of \(p\) as a parameter controlling the trade-off between feature selectivity and uniformity.
\begin{figure}
    \centering
    \includegraphics[width=1\linewidth]{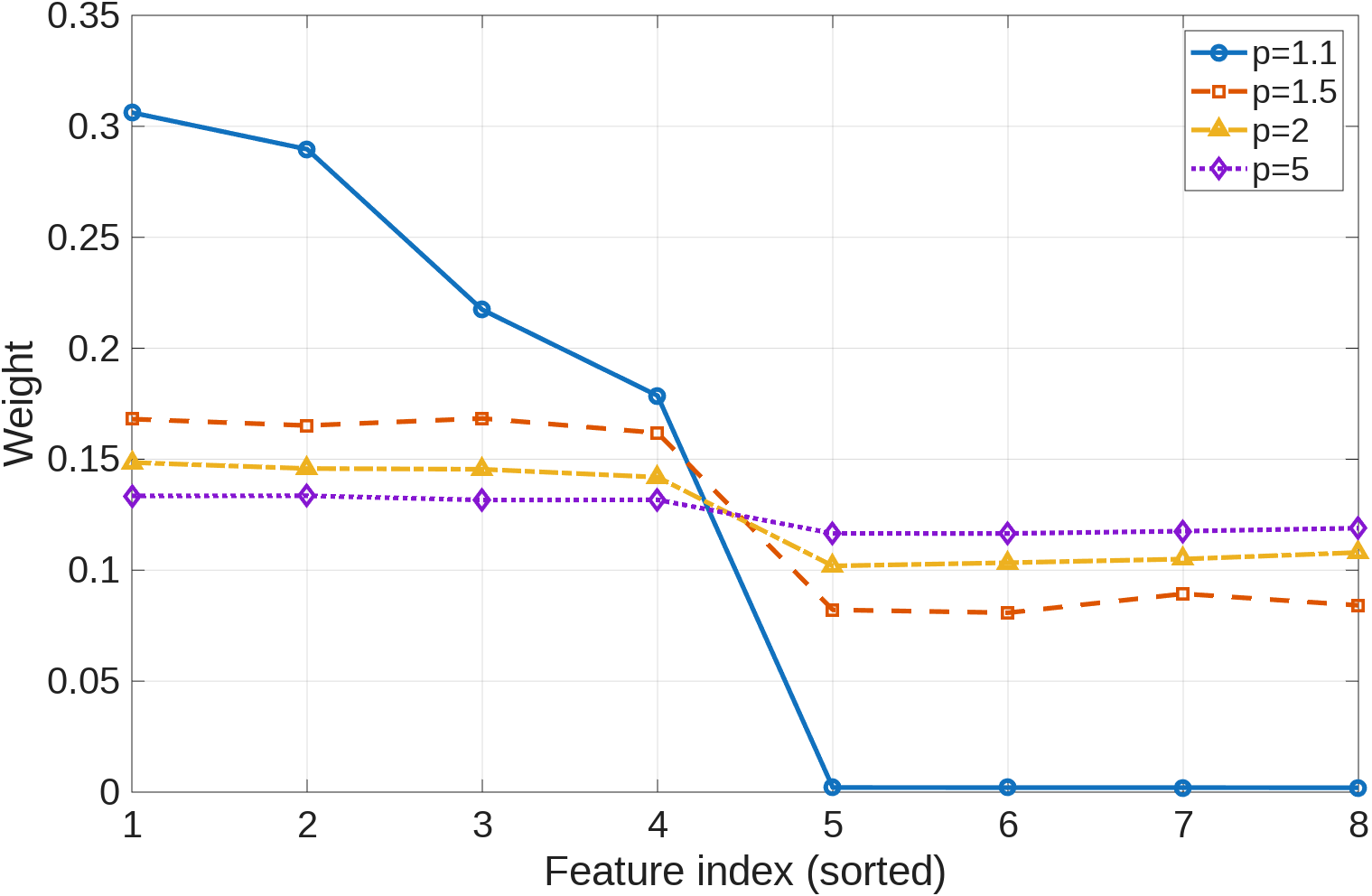}
    \caption{Sorted feature weights for different values of \(p\). Smaller \(p\) yields sparse weight distributions, while larger \(p\) produces near-uniform weights, consistent with Proposition~\ref{prop:weight_structure}.}
    \label{fig:weights_stability}
\end{figure}

Figure~\ref{fig:objective_within_bounds} illustrates the behaviour of the $mwk$-means objective relative to the bounds established in Theorem~\ref{thm:mwk_bounds}. To enable comparison across data sets and runs, we report a normalised version of the objective, obtained by linearly scaling it between its theoretical lower and upper bounds. In this representation, a value of $0$ corresponds to the lower bound, while a value of $1$ corresponds to the upper bound. All observed values lie within the interval $[0,1]$, confirming that the empirical objective satisfies the theoretical bounds for all tested values of $p$.
\begin{figure}
    \centering
    \includegraphics[width=1\linewidth]{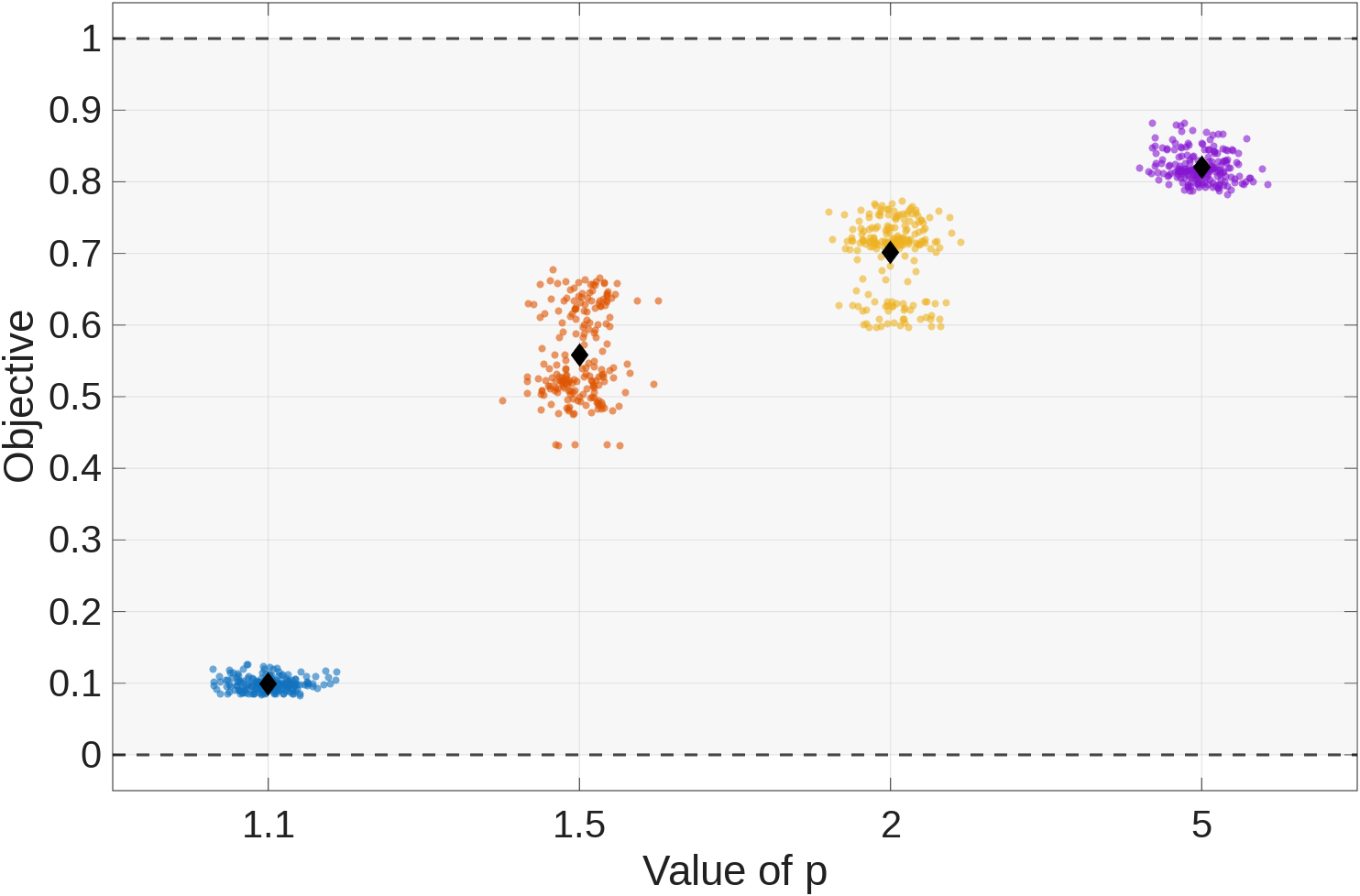}
    \caption{Values of the normalised objective across data sets and runs for different values of \(p\). The shaded region denotes the theoretical bounds \([0,1]\) established in Theorem~\ref{thm:mwk_bounds}. All observed values lie within this interval, confirming that the objective remains within the derived bounds as \(p\) varies.}
    \label{fig:objective_within_bounds}
\end{figure}

The normalised objective remains within the theoretical bounds, with mean values in the range $[0.10,\,0.82]$ across all tested values of $p$. This behaviour is consistent with Theorem~\ref{thm:mwk_bounds}, which shows that the lower and upper bounds are attained in the limits $p \to 1^+$ and $p \to \infty$, respectively.

Together, these results provide empirical support for both the structural properties of the weighting scheme and the theoretical bounds on the objective.

\section{Cluster-Adaptive Feature Extraction}
In this section, we introduce the Cluster-Adaptive Feature Extraction (CAFE) method. Its main objective is to facilitate the extraction of meaningful features from data sets containing cluster-specific noise. We start with a definition.
\begin{dfn}
\label{def::noisepair}
A (cluster, feature) pair $(l,v)$ is a \emph{noise pair} if $D_{lv} > D_{lv}^0$, where $D_{lv}^0$ denotes the dispersion of feature $v$ in cluster $S_l$ in the absence of noise corruption.
\end{dfn}

This definition formalises the assumption that noise manifests as inflated within-cluster dispersion. Since $mwk$-means assigns lower weights to features with higher within-cluster dispersion, noisy features will receive lower weights than informative ones. This motivates rescaling $X$ by these weights, as described below.
\begin{dfn}
    \label{def:rescale}
    Let \(\Tilde{X}=(\Tilde{x}_{iv})\) be the re-scaled version of \(X\). That is, if \(x_i \in S_l\) then \(\Tilde{x}_{iv}=w_{lv}x_{iv}\).
\end{dfn}

\begin{lemma}
\label{lemma:dispersion_reversal}
Within each cluster, the feature ordering by dispersion in $\tilde{X}$ is the reverse of that in $X$.
\end{lemma}
\begin{proof}
By Definition~\ref{def:rescale}, $\tilde{x}_{iv} = w_{lv} x_{iv}$ for $x_i \in S_l$. The Minkowski centre $\tilde{z}_{lv}$ minimises $\sum_{x_i \in S_l} |w_{lv} x_{iv} - z|^p$. Substituting $z = w_{lv} z'$ gives
\[
\sum_{x_i \in S_l} |w_{lv} x_{iv} - w_{lv} z'|^p = w_{lv}^p \sum_{x_i \in S_l} |x_{iv} - z'|^p,
\]
which is minimised at $z' = z_{lv}$. Hence $\tilde{z}_{lv} = w_{lv} z_{lv}$, and
\[
\tilde{D}_{lv} = \sum_{x_i \in S_l} |w_{lv} x_{iv} - w_{lv} z_{lv}|^p = w_{lv}^p D_{lv}.
\]
For any two features $u, v$ in cluster $S_l$, by Proposition~\ref{prop:weight_structure},
\[
\frac{w_{lu}}{w_{lv}} = \left(\frac{D_{lv}}{D_{lu}}\right)^{\frac{1}{p-1}},
\]
so
\[
\frac{\tilde{D}_{lu}}{\tilde{D}_{lv}} = \left(\frac{w_{lu}}{w_{lv}}\right)^p \cdot \frac{D_{lu}}{D_{lv}} = \left(\frac{D_{lv}}{D_{lu}}\right)^{\frac{p}{p-1}} \cdot \frac{D_{lu}}{D_{lv}} = \left(\frac{D_{lu}}{D_{lv}}\right)^{-\frac{p}{p-1}}\cdot\left(\frac{D_{lu}}{D_{lv}}\right)^1=\left(\frac{D_{lu}}{D_{lv}}\right)^{-\frac{1}{p-1}}.
\]
Since $-1/(p-1) < 0$ for $p > 1$, $D_{lu} > D_{lv}$ implies $\tilde{D}_{lu} < \tilde{D}_{lv}$.
\end{proof}
\begin{cor}
\label{cor:noise_suppression}
Let $(l,u)$ be a noise pair such that $D_{lu} \geq C\,D_{lv}$ for all $v \neq u$, with $C > 1$. Then
\[
\tilde{D}_{lu} \;\leq\; \frac{D_{lu}}{\left(1 + (m-1) C^{\frac{1}{p-1}}\right)^p}.
\]
\end{cor}
\begin{proof}
By Lemma~\ref{lemma:dispersion_reversal}, $\tilde{D}_{lu} = w_{lu}^p D_{lu}$. By Theorem~\ref{thm:suppression},
\[
w_{lu} \leq \frac{1}{1 + (m-1) C^{\frac{1}{p-1}}}.
\]
Raising to the power $p$ and multiplying by $D_{lu}$ gives the result.
\end{proof}

Together, Lemma~\ref{lemma:dispersion_reversal} and Corollary~\ref{cor:noise_suppression} characterise the effect of the rescaling on within-cluster dispersion. The lemma shows that the rescaling reverses the dispersion ordering within each cluster, so that informative features (which had low dispersion in $X$) now have high dispersion in $\tilde{X}$. The corollary quantifies this effect for dominant noise pairs: when a noisy feature has dispersion at least $C$ times larger than all others in its cluster, its dispersion in $\tilde{X}$ is bounded above by a quantity that decreases as either $C$ or $p$ increases, reflecting stronger suppression for more dominant noise pairs and more selective values of $p$. Under our assumption that high within-cluster dispersion indicates corruption (Definition~\ref{def::noisepair}), this means that $\tilde{X}$ is a more informative input to feature extraction methods such as PCA. Since PCA finds directions of maximum variance, it will be driven by informative rather than noisy features in $\tilde{X}$, leading to components that better reflect the cluster structure of the data. Algorithm~\ref{alg:cafe} describes the steps of CAFE.

\begin{algorithm}
\caption{Cluster-Adaptive Feature Extraction (CAFE)}
\label{alg:cafe}
\begin{algorithmic}[1]
    \Require Data set $X$, number of clusters $k$, number of 
    extracted features $r$, feature extraction method $\mathcal{F}$.
    \Ensure Extracted feature matrix $\mathcal{F}(\tilde{X}) \in 
    \mathbb{R}^{n \times r}$.
    \For{$p \in \{1.1, 1.5, 2, 3\}$}
        \For{$i = 1$ until $50$}
            \State Run $mwk$-means on $X$ with exponent $p$.
        \EndFor
        \State Retain the solution $(\hat{S}^{(p)}, \hat{W}^{(p)})$ with the lowest objective value of (\ref{eq:mwk}).
    \EndFor
    \State Select as final clustering that with the highest ARI to all others. That is, find
    \[
    p^* = \arg\max_p \sum_{q \neq p} 
    \mathrm{ARI}\!\left(\hat{S}^{(p)}, \hat{S}^{(q)}\right),
    \]
    and set $S = \hat{S}^{(p^*)}$ and $W = \hat{W}^{(p^*)}$.
    \State Construct $\tilde{X}$ by rescaling each $x_i \in S_l$ 
    feature-wise with
    \[
    \tilde{x}_{iv} = w_{lv}\, x_{iv}.
    \]
    \State \Return $\mathcal{F}(\tilde{X})$, the leading $r$ components 
    of the chosen feature extraction method applied to $\tilde{X}$.
\end{algorithmic}
\end{algorithm}

\section{Experiments}

In this section, we evaluate the performance of CAFE empirically. We compare three CAFE variants against five baseline feature extraction methods on a collection of UCI benchmark data sets corrupted with controlled within-cluster noise, directly instantiating the setting described in 
Definition~\ref{def::noisepair}. Our aim is to assess whether the cluster-adaptive rescaling of CAFE leads to improved clustering quality after feature extraction, and to identify the conditions under which this improvement is most pronounced.

\subsection{Experimental setup}

We evaluate CAFE on a collection of benchmark data sets from the UCI Machine Learning Repository~\cite{Dua:2019}. We converted categorical features to numerical using one-hot encoding prior to any further processing. That is, each categorical feature was replaced with a set of binary indicator features (one per category value).

To evaluate robustness to within-cluster noise, we corrupted each data set using six configurations. We obtained these by crossing three noise magnitudes $\sigma \in \{1, 2, 5\}$ with two noise levels $\eta \in \{20\%, 40\%\}$ using the following procedure: a set of $\max(1, \text{round}(\eta \cdot k \cdot m))$ (cluster, feature) pairs is selected uniformly at random from all $k \times m$ possible pairs, and Gaussian noise $\mathcal{N}(0,\sigma^2)$ is added to the values of the selected feature within the selected cluster. 

The above led to 54 base data sets. For each of these base data sets, we generated 20 independent noisy versions, to a total of 1,080 data sets. We then normalised each noisy data set feature-wise with
\begin{equation}
x_{iv} = \frac{x_{iv} - \bar{x}_v}
{\max_i x_{iv} - \min_i x_{iv}},
\end{equation}
where $\bar{x}_v$ is the mean of feature $v$ over the full data set.

In our experiments, we set all methods to extract the same number of features, $r = \lfloor m/2 \rfloor$. Afterwards, to keep in line with the unsupervised nature of our work, we ran $k$-means 100 times using the extracted features, and selected as final clustering that with the lowest $k$-means objective. Finally, we calculated the Adjusted Rand Index (ARI) between this final clustering and the ground truth of each data set.

We experimented with three CAFE variants, CAFE+PCA, CAFE+NMF, and CAFE+ICA, corresponding to using PCA, NMF, or ICA as the downstream feature extraction method $\mathcal{F}$ (see Algorithm \ref{alg:cafe}). These are compared against five baselines that apply each feature extraction 
method directly to the original data $X$ without rescaling: PCA, NMF (50 replicates, best reconstruction), ICA (50 random initialisations, best reconstruction), a single-layer autoencoder (50 random initialisations, 100 training epochs, best mean squared reconstruction error), and UMAP.

\subsection{Results and Discussion}

Tables~\ref{tab:results_20} and~\ref{tab:results_40} report the mean ARI and standard deviation over the 20 noisy versions for each data set, noise magnitude $\sigma$, and noise level $\eta$. For each CAFE/baseline pair (CAFE+PCA vs PCA, CAFE+NMF vs NMF, CAFE+ICA vs ICA), the better result is shown in bold. For each row, we also underlined the best overall result. Across both noise levels and all values of $\sigma$, CAFE outperforms its corresponding baseline in 48 out of 54 configurations.

\begin{table}[p]
\centering
\caption{Mean ARI $\pm$ standard deviation over 20 noisy versions of the data sets for $\eta = 20\%$. Bold indicates the best result within each CAFE/baseline pair. Underline indicates the best result in a row.}
\label{tab:results_20}
\resizebox{\linewidth}{!}{
\begin{tabular}{llllllllll}
\toprule
Data set & $\sigma$ & Auto & UMAP & PCA & CAFE+PCA & NMF & CAFE+NMF & ICA & CAFE+ICA \\
\midrule
 \multirow{3}{*}{AustraCC} & 1 & $0.032 \pm 0.12$ & $0.175 \pm 0.21$ & $\underline{\mathbf{0.273}} \pm 0.24$ & $0.131 \pm 0.19$ & $\mathbf{0.259} \pm 0.25$ & $0.151 \pm 0.20$ & $\underline{\mathbf{0.273}} \pm 0.24$ & $0.218 \pm 0.24$ \\
  & 2 & $0.013 \pm 0.05$ & $0.162 \pm 0.21$ & $\underline{\mathbf{0.217}} \pm 0.21$ & $0.126 \pm 0.17$ & $\mathbf{0.170 }\pm 0.22$ & $0.153 \pm 0.20$ & $\underline{\mathbf{0.217}} \pm 0.21$ & $0.183 \pm 0.21$ \\
  & 5 & $0.080 \pm 0.18$ & $0.195 \pm 0.21$ & $\mathbf{0.321} \pm 0.22$ & $0.295 \pm 0.26$ & $0.195 \pm 0.22$ & $\mathbf{0.270} \pm 0.25$ & $0.321 \pm 0.22$ & $\underline{\mathbf{0.332}} \pm 0.24$ \\
\cmidrule{1-10}
 \multirow{3}{*}{Balance} & 1 & $0.022 \pm 0.02$ & $0.038 \pm 0.04$ & $0.025 \pm 0.03$ & $\underline{\mathbf{0.061}} \pm 0.04$ & $0.024 \pm 0.03$ & $\underline{\mathbf{0.061}} \pm 0.04$ & $0.024 \pm 0.02$ & $\mathbf{0.049} \pm 0.04$ \\
  & 2 & $0.018 \pm 0.02$ & $0.031 \pm 0.03$ & $0.023 \pm 0.03$ & $\mathbf{0.042} \pm 0.03$ & $0.029 \pm 0.03$ & $\mathbf{0.048} \pm 0.03$ & $0.031 \pm 0.03$ & $\underline{\mathbf{0.056}} \pm 0.05$ \\
  & 5 & $0.021 \pm 0.03$ & $0.043 \pm 0.04$ & $0.025 \pm 0.03$ & $\mathbf{0.064} \pm 0.05$ & $0.019 \pm 0.03$ & $\underline{\mathbf{0.072}} \pm 0.07$ & $0.025 \pm 0.03$ & $\mathbf{0.065} \pm 0.06$ \\
\cmidrule{1-10}
 \multirow{3}{*}{\shortstack{Breast\\Cancer}} & 1 & $0.823 \pm 0.04$ & $0.797 \pm 0.16$ & $0.799 \pm 0.05$ & $\underline{\mathbf{0.827}} \pm 0.08$ & $0.752 \pm 0.08$ & $\mathbf{0.785} \pm 0.09$ & $0.794 \pm 0.05$ & $\mathbf{0.825} \pm 0.08$ \\
  & 2 & $0.822 \pm 0.03$ & $0.760 \pm 0.21$ & $0.791 \pm 0.05$ & $\underline{\mathbf{0.833}} \pm 0.06$ & $0.752 \pm 0.09$ & $\mathbf{0.801} \pm 0.05$ & $0.789 \pm 0.05$ & $\mathbf{0.832} \pm 0.06$ \\
  & 5 & $0.826 \pm 0.03$ & $0.821 \pm 0.09$ & $0.808 \pm 0.03$ & $\mathbf{0.861} \pm 0.03$ & $0.798 \pm 0.04$ & $\mathbf{0.829} \pm 0.05$ & $0.809 \pm 0.03$ & $\underline{\mathbf{0.862}} \pm 0.03$ \\
\cmidrule{1-10}
 \multirow{3}{*}{CarEvaluation} & 1 & $0.027 \pm 0.03$ & $0.028 \pm 0.06$ & $\mathbf{0.060} \pm 0.05$ & $0.043 \pm 0.07$ & $\underline{\mathbf{0.066}} \pm 0.06$ & $0.064 \pm 0.07$ & $\mathbf{0.052} \pm 0.05$ & $0.037 \pm 0.08$ \\
  & 2 & $0.037 \pm 0.06$ & $0.040 \pm 0.04$ & $0.060 \pm 0.06$ & $\underline{\mathbf{0.071}} \pm 0.06$ & $\mathbf{0.051} \pm 0.05$ & $0.042 \pm 0.07$ & $\mathbf{0.056} \pm 0.05$ & $0.055 \pm 0.07$ \\
  & 5 & $0.021 \pm 0.04$ & $0.022 \pm 0.04$ & $0.053 \pm 0.05$ & $\mathbf{0.066} \pm 0.07$ & $0.036 \pm 0.04$ & $\mathbf{0.070} \pm 0.11$ & $0.037 \pm 0.04$ & $\underline{\mathbf{0.088}} \pm 0.12$ \\
\cmidrule{1-10}
 \multirow{3}{*}{Ecoli} & 1 & $0.110 \pm 0.07$ & $0.279 \pm 0.05$ & $0.196 \pm 0.09$ & $\underline{\mathbf{0.332}} \pm 0.18$ & $0.177 \pm 0.08$ & $\mathbf{0.294} \pm 0.15$ & $0.205 \pm 0.10$ & $\mathbf{0.325} \pm 0.16$ \\
  & 2 & $0.152 \pm 0.08$ & $0.298 \pm 0.05$ & $0.220 \pm 0.09$ & $\underline{\mathbf{0.403}} \pm 0.17$ & $0.212 \pm 0.11$ & $\mathbf{0.312} \pm 0.10$ & $0.253 \pm 0.12$ & $\mathbf{0.370} \pm 0.14$ \\
  & 5 & $0.141 \pm 0.11$ & $0.260 \pm 0.08$ & $0.176 \pm 0.10$ & $\underline{\mathbf{0.355}} \pm 0.16$ & $0.158 \pm 0.11$ & $\mathbf{0.303} \pm 0.12$ & $0.192 \pm 0.12$ & $\mathbf{0.328} \pm 0.13$ \\
\cmidrule{1-10}
 \multirow{3}{*}{Glass} & 1 & $0.096 \pm 0.09$ & $0.223 \pm 0.06$ & $0.165 \pm 0.07$ & $\mathbf{0.236} \pm 0.08$ & $0.130 \pm 0.07$ & $\underline{\mathbf{0.266}} \pm 0.10$ & $0.154 \pm 0.10$ & $\mathbf{0.261} \pm 0.09$ \\
  & 2 & $0.090 \pm 0.10$ & $0.208 \pm 0.07$ & $0.136 \pm 0.06$ & $\mathbf{0.261} \pm 0.13$ & $0.134 \pm 0.08$ & $\underline{\mathbf{0.278}} \pm 0.14$ & $0.150 \pm 0.08$ & $\mathbf{0.273} \pm 0.13$ \\
  & 5 & $0.115 \pm 0.09$ & $0.233 \pm 0.10$ & $0.135 \pm 0.06$ & $\mathbf{0.247} \pm 0.14$ & $0.114 \pm 0.07$ & $\mathbf{0.254} \pm 0.10$ & $0.121 \pm 0.07$ & $\underline{\mathbf{0.257}} \pm 0.13$ \\
\cmidrule{1-10}
 \multirow{3}{*}{Ionosphere} & 1 & $-0.018 \pm 0.03$ & $0.132 \pm 0.03$ & $\underline{\mathbf{0.152}} \pm 0.02$ & $0.148 \pm 0.02$ & $0.085 \pm 0.04$ & $\mathbf{0.106} \pm 0.04$ & $\mathbf{0.151} \pm 0.02$ & $0.142 \pm 0.01$ \\
  & 2 & $-0.026 \pm 0.02$ & $0.129 \pm 0.03$ & $0.146 \pm 0.02$ & $\underline{\mathbf{0.148}} \pm 0.05$ & $0.078 \pm 0.04$ & $\mathbf{0.099} \pm 0.04$ & $0.135 \pm 0.02$ & $\mathbf{0.146} \pm 0.05$ \\
  & 5 & $-0.024 \pm 0.02$ & $0.133 \pm 0.02$ & $0.139 \pm 0.02$ & $\mathbf{0.184} \pm 0.19$ & $0.075 \pm 0.05$ & $\mathbf{0.166} \pm 0.12$ & $0.134 \pm 0.02$ & $\underline{\mathbf{0.187}} \pm 0.19$ \\
\cmidrule{1-10}
 \multirow{3}{*}{Iris} & 1 & $0.450 \pm 0.06$ & $0.772 \pm 0.11$ & $0.708 \pm 0.10$ & $\mathbf{0.728} \pm 0.11$ & $0.750 \pm 0.13$ & $\underline{\mathbf{0.781}} \pm 0.11$ & $0.684 \pm 0.16$ & $\mathbf{0.756} \pm 0.13$ \\
  & 2 & $0.447 \pm 0.11$ & $\underline{0.790} \pm 0.12$ & $0.682 \pm 0.10$ & $\mathbf{0.701} \pm 0.16$ & $\mathbf{0.735} \pm 0.12$ & $0.732 \pm 0.16$ & $0.634 \pm 0.15$ & $\mathbf{0.714} \pm 0.17$ \\
  & 5 & $0.449 \pm 0.06$ & $\underline{0.745} \pm 0.13$ & $\mathbf{0.692} \pm 0.11$ & $0.597 \pm 0.21$ & $\mathbf{0.737} \pm 0.11$ & $0.696 \pm 0.18$ & $\mathbf{0.644} \pm 0.14$ & $0.621 \pm 0.23$ \\
\cmidrule{1-10}
 \multirow{3}{*}{\shortstack{Teaching\\Assistant}} & 1 & $0.026 \pm 0.01$ & $0.026 \pm 0.01$ & $0.027 \pm 0.01$ & $\underline{\mathbf{0.072}} \pm 0.08$ & $0.024 \pm 0.01$ & $\mathbf{0.050} \pm 0.05$ & $0.029 \pm 0.01$ & $\mathbf{0.056} \pm 0.06$ \\
  & 2 & $0.025 \pm 0.02$ & $0.021 \pm 0.01$ & $0.027 \pm 0.01$ & $\underline{\mathbf{0.183}} \pm 0.17$ & $0.022 \pm 0.01$ & $\mathbf{0.088} \pm 0.11$ & $0.022 \pm 0.01$ & $\mathbf{0.118} \pm 0.11$ \\
  & 5 & $0.025 \pm 0.01$ & $0.024 \pm 0.01$ & $0.025 \pm 0.01$ & $\underline{\mathbf{0.162}} \pm 0.16$ & $0.022 \pm 0.01$ & $\mathbf{0.151} \pm 0.13$ & $0.025 \pm 0.01$ & $\mathbf{0.151} \pm 0.14$ \\
\bottomrule
\end{tabular}
}
\end{table}

At $\eta = 20\%$ (Table \ref{tab:results_20}), CAFE improves over its corresponding baselines in 22 out of 27 configurations. The five exceptions occur at lower noise magnitudes: Australian Credit Card and Car Evaluation at $\sigma \in \{1, 2\}$, Ionosphere at $\sigma = 1$, and Iris at $\sigma = 5$. This is consistent with the theoretical justification of CAFE: when the noise magnitude is small, the inflation of within-cluster dispersion at noise pairs may not be large enough to clearly separate them from informative features, reducing the effectiveness of the rescaling. It is also worth noting that several cells exhibit large standard deviations relative to the mean, reflecting the variability introduced by the random selection of noise pairs across the 20 realisations. This is particularly evident for Ionosphere at $\sigma = 5$, where the mean ARI and standard deviation are high, suggesting that the outcome depends strongly on which (cluster, feature) pairs happen to be corrupted.

\begin{table}[p]
\centering
\caption{Mean ARI $\pm$ standard deviation over 20 noisy versions of the data sets for $\eta = 40\%$. Bold indicates the best result within each CAFE/baseline pair. Underline indicates the best result in a row.}
\label{tab:results_40}
\resizebox{\linewidth}{!}{
\begin{tabular}{llll|ll|ll|ll}
\toprule
Data set & $\sigma$ & Auto & UMAP & PCA & CAFE+PCA & NMF & CAFE+NMF & ICA & CAFE+ICA \\
\midrule
 \multirow{3}{*}{AustraCC} & 1 & $0.064 \pm 0.14$ & $0.121 \pm 0.19$ & $0.181 \pm 0.21$ & $\underline{\mathbf{0.424}} \pm 0.22$ & $0.125 \pm 0.20$ & $\mathbf{0.266} \pm 0.21$ & $0.166 \pm 0.20$ & $\mathbf{0.409} \pm 0.25$ \\
  & 2 & $0.086 \pm 0.16$ & $0.137 \pm 0.20$ & $0.199 \pm 0.24$ & $\underline{\mathbf{0.446}} \pm 0.22$ & $0.124 \pm 0.21$ & $\mathbf{0.218} \pm 0.23$ & $0.199 \pm 0.24$ & $\mathbf{0.400} \pm 0.24$ \\
  & 5 & $0.046 \pm 0.12$ & $0.126 \pm 0.18$ & $0.235 \pm 0.24$ & $\underline{\mathbf{0.492}} \pm 0.27$ & $0.074 \pm 0.16$ & $\mathbf{0.322} \pm 0.31$ & $0.235 \pm 0.24$ & $\mathbf{0.465} \pm 0.28$ \\
\cmidrule{1-10}
 \multirow{3}{*}{Balance} & 1 & $0.029 \pm 0.03$ & $0.018 \pm 0.05$ & $0.021 \pm 0.03$ & $\mathbf{0.075} \pm 0.09$ & $0.022 \pm 0.03$ & $\underline{\mathbf{0.100}} \pm 0.11$ & $0.022 \pm 0.03$ & $\mathbf{0.083} \pm 0.09$ \\
  & 2 & $0.018 \pm 0.02$ & $0.027 \pm 0.04$ & $0.028 \pm 0.03$ & $\mathbf{0.097} \pm 0.11$ & $0.019 \pm 0.03$ & $\mathbf{0.126} \pm 0.13$ & $0.023 \pm 0.03$ & $\underline{\mathbf{0.143}} \pm 0.17$ \\
  & 5 & $0.032 \pm 0.03$ & $0.013 \pm 0.02$ & $0.033 \pm 0.04$ & $\underline{\mathbf{0.245}} \pm 0.32$ & $0.021 \pm 0.03$ & $\mathbf{0.144} \pm 0.20$ & $0.016 \pm 0.02$ & $\mathbf{0.143} \pm 0.21$ \\
\cmidrule{1-10}
 \multirow{3}{*}{\shortstack{Breast\\Cancer}} & 1 & $0.792 \pm 0.07$ & $0.825 \pm 0.03$ & $0.723 \pm 0.08$ & $\underline{\mathbf{0.837}} \pm 0.03$ & $0.659 \pm 0.11$ & $\mathbf{0.792} \pm 0.06$ & $0.720 \pm 0.08$ & $\mathbf{0.835} \pm 0.03$ \\
  & 2 & $0.781 \pm 0.06$ & $0.790 \pm 0.13$ & $0.701 \pm 0.08$ & $\underline{\mathbf{0.798}} \pm 0.12$ & $0.676 \pm 0.09$ & $\mathbf{0.776} \pm 0.09$ & $0.701 \pm 0.08$ & $\underline{\mathbf{0.798}} \pm 0.12$ \\
  & 5 & $0.787 \pm 0.05$ & $\underline{0.823} \pm 0.04$ & $0.718 \pm 0.08$ & $\mathbf{0.819} \pm 0.04$ & $0.670 \pm 0.10$ & $\mathbf{0.786} \pm 0.06$ & $0.715 \pm 0.08$ & $\mathbf{0.818} \pm 0.04$ \\
\cmidrule{1-10}
 \multirow{3}{*}{CarEvaluation} & 1 & $0.020 \pm 0.04$ & $0.023 \pm 0.03$ & $0.022 \pm 0.02$ & $\mathbf{0.037} \pm 0.04$ & $0.029 \pm 0.03$ & $\underline{\mathbf{0.042}} \pm 0.04$ & $0.027 \pm 0.03$ & $\mathbf{0.040} \pm 0.05$ \\
  & 2 & $0.009 \pm 0.04$ & $0.031 \pm 0.04$ & $0.019 \pm 0.03$ & $\underline{\mathbf{0.097}} \pm 0.12$ & $0.017 \pm 0.03$ & $\mathbf{0.046} \pm 0.05$ & $0.022 \pm 0.04$ & $\mathbf{0.076} \pm 0.09$ \\
  & 5 & $0.017 \pm 0.04$ & $0.047 \pm 0.10$ & $0.018 \pm 0.04$ & $\underline{\mathbf{0.251}} \pm 0.27$ & $0.013 \pm 0.03$ & $\mathbf{0.169} \pm 0.28$ & $0.018 \pm 0.04$ & $\mathbf{0.173} \pm 0.27$ \\
\cmidrule{1-10}
 \multirow{3}{*}{Ecoli} & 1 & $0.053 \pm 0.07$ & $0.197 \pm 0.06$ & $0.116 \pm 0.11$ & $\underline{\mathbf{0.352}} \pm 0.19$ & $0.086 \pm 0.10$ & $\mathbf{0.295} \pm 0.16$ & $0.113 \pm 0.11$ & $\mathbf{0.263} \pm 0.13$ \\
  & 2 & $0.028 \pm 0.04$ & $0.186 \pm 0.05$ & $0.082 \pm 0.06$ & $\underline{\mathbf{0.339}} \pm 0.19$ & $0.056 \pm 0.05$ & $\mathbf{0.261} \pm 0.15$ & $0.093 \pm 0.06$ & $\mathbf{0.318} \pm 0.14$ \\
  & 5 & $0.023 \pm 0.04$ & $0.166 \pm 0.05$ & $0.077 \pm 0.08$ & $\underline{\mathbf{0.324}} \pm 0.22$ & $0.050 \pm 0.06$ & $\mathbf{0.249} \pm 0.21$ & $0.074 \pm 0.05$ & $\mathbf{0.245} \pm 0.16$ \\
\cmidrule{1-10}
 \multirow{3}{*}{Glass} & 1 & $0.032 \pm 0.05$ & $0.154 \pm 0.07$ & $0.098 \pm 0.07$ & $\mathbf{0.240} \pm 0.11$ & $0.076 \pm 0.06$ & $\mathbf{0.223} \pm 0.09$ & $0.098 \pm 0.06$ & $\underline{\mathbf{0.243}} \pm 0.10$ \\
  & 2 & $0.014 \pm 0.03$ & $0.162 \pm 0.07$ & $0.083 \pm 0.05$ & $\mathbf{0.253} \pm 0.09$ & $0.064 \pm 0.03$ & $\mathbf{0.252} \pm 0.13$ & $0.089 \pm 0.05$ & $\underline{\mathbf{0.258}} \pm 0.16$ \\
  & 5 & $0.032 \pm 0.07$ & $0.154 \pm 0.09$ & $0.081 \pm 0.05$ & $\underline{\mathbf{0.253}} \pm 0.18$ & $0.068 \pm 0.05$ & $\mathbf{0.198} \pm 0.14$ & $0.078 \pm 0.07$ & $\mathbf{0.229} \pm 0.19$ \\
\cmidrule{1-10}
 \multirow{3}{*}{Ionosphere} & 1 & $-0.005 \pm 0.04$ & $0.121 \pm 0.04$ & $0.127 \pm 0.03$ & $\mathbf{0.134} \pm 0.03$ & $0.056 \pm 0.04$ & $\mathbf{0.067} \pm 0.04$ & $0.123 \pm 0.03$ & $\underline{\mathbf{0.138}} \pm 0.03$ \\
  & 2 & $-0.006 \pm 0.04$ & $0.125 \pm 0.04$ & $0.124 \pm 0.03$ & $\mathbf{0.209} \pm 0.31$ & $0.052 \pm 0.04$ & $\mathbf{0.201} \pm 0.23$ & $0.124 \pm 0.03$ & $\underline{\mathbf{0.270}} \pm 0.35$ \\
  & 5 & $0.027 \pm 0.07$ & $0.116 \pm 0.05$ & $0.115 \pm 0.03$ & $\underline{\mathbf{0.680}} \pm 0.36$ & $0.063 \pm 0.04$ & $\mathbf{0.437} \pm 0.36$ & $0.112 \pm 0.04$ & $\mathbf{0.641} \pm 0.37$ \\
\cmidrule{1-10}
 \multirow{3}{*}{Iris} & 1 & $0.306 \pm 0.23$ & $0.429 \pm 0.22$ & $0.303 \pm 0.14$ & $\underline{\mathbf{0.649}} \pm 0.20$ & $0.293 \pm 0.22$ & $\mathbf{0.556} \pm 0.19$ & $0.304 \pm 0.14$ & $\mathbf{0.558} \pm 0.21$ \\
  & 2 & $0.338 \pm 0.32$ & $0.443 \pm 0.27$ & $0.328 \pm 0.26$ & $\underline{\mathbf{0.567}} \pm 0.24$ & $0.307 \pm 0.26$ & $\mathbf{0.495} \pm 0.25$ & $0.294 \pm 0.19$ & $\mathbf{0.505} \pm 0.24$ \\
  & 5 & $0.349 \pm 0.33$ & $\underline{0.466} \pm 0.20$ & $0.357 \pm 0.28$ & $\mathbf{0.422} \pm 0.27$ & $0.355 \pm 0.29$ & $\mathbf{0.423} \pm 0.29$ & $0.281 \pm 0.21$ & $\mathbf{0.382} \pm 0.24$ \\
\cmidrule{1-10}
 \multirow{3}{*}{\shortstack{Teaching\\Assistant}} & 1 & $0.019 \pm 0.02$ & $0.023 \pm 0.02$ & $0.023 \pm 0.02$ & $\mathbf{0.120} \pm 0.15$ & $0.013 \pm 0.02$ & $\mathbf{0.050} \pm 0.10$ & $0.023 \pm 0.02$ & $\underline{\mathbf{0.194}} \pm 0.18$ \\
  & 2 & $0.022 \pm 0.02$ & $0.039 \pm 0.02$ & $0.020 \pm 0.01$ & $\underline{\mathbf{0.167}} \pm 0.18$ & $0.021 \pm 0.01$ & $\mathbf{0.099} \pm 0.10$ & $0.019 \pm 0.01$ & $\mathbf{0.123} \pm 0.17$ \\
  & 5 & $0.012 \pm 0.01$ & $0.042 \pm 0.06$ & $0.015 \pm 0.01$ & $\underline{\mathbf{0.182}} \pm 0.19$ & $0.013 \pm 0.01$ & $\mathbf{0.040} \pm 0.04$ & $0.015 \pm 0.01$ & $\mathbf{0.148} \pm 0.19$ \\
\bottomrule
\end{tabular}
}
\end{table}

At $\eta = 40\%$ (Table \ref{tab:results_40}), CAFE outperforms its corresponding baselines on all 27 configurations. The improvement over the baselines is also generally larger at this noise level, reflecting the fact that a higher proportion of corrupted pairs creates a clearer dispersion signal for the rescaling to exploit. Together, the two tables show a consistent pattern: the advantage of CAFE over its baselines grows with both the noise level $\eta$ and the noise magnitude $\sigma$. The Autoencoder and UMAP baselines are generally weaker than PCA, NMF, and ICA across both tables, suggesting that these methods are less suited to recovering cluster structure from noisy data in this setting. Among the three CAFE variants, no single one dominates consistently, though CAFE+PCA and CAFE+ICA tend to show the largest absolute improvements over their respective baselines.

\section{Conclusion}

In this paper, we provide a theoretical analysis of the $mwk$-means algorithm and introduce CAFE, a new method for cluster-adaptive feature extraction. On the theoretical side, we show that the $mwk$-means objective function can be expressed as a power-mean aggregation of within-cluster dispersions, yielding a unified interpretation of the role of the Minkowski exponent. This formulation enables us to derive bounds for the objective, characterise the structure and scaling behaviour of the feature weights, and establish convergence guarantees for the algorithm. 
This perspective clarifies how the choice of the exponent governs the trade-off between feature selectivity and uniformity, and explains the mechanism by which features with higher dispersion are suppressed. Features that consistently receive low weights across a range of values of $p$ may be interpreted as robustly irrelevant across different distance geometries and could be considered for removal prior to clustering.

Building on this theoretical foundation, we introduce CAFE that uses the feature weights produced by $mwk$-means to rescale the data prior to unsupervised feature extraction. We prove that this rescaling reverses the within-cluster dispersion ordering, so that noisy features --- which inflate within-cluster dispersion --- receive lower dispersion in the rescaled data, while informative features are amplified. This provides a principled justification for applying standard feature extraction methods such as PCA, NMF, or ICA to the rescaled data. Experiments on nine UCI benchmark data sets corrupted with controlled within-cluster noise show that CAFE improves clustering quality over five baseline methods in 48 out of 54 experimental configurations, with the improvement growing consistently with both noise magnitude and noise level.

Future work may explore extensions of CAFE to other clustering formulations, investigate the sensitivity of the method to the choice of $p$, and study its behaviour on data sets where the assumption linking high within-cluster dispersion to noise is only partially satisfied.

\bibliographystyle{ieeetr}
\bibliography{references}

\end{document}